\documentclass{article}
\PassOptionsToPackage{sort&compress}{natbib}
\usepackage[final]{corl_2026} %
\usepackage[numbers]{natbib}
\usepackage{multicol}
\usepackage{subfigure}
\usepackage{xcolor}
\usepackage{times}
\usepackage{epsfig}
\usepackage{amsmath}
\usepackage{amssymb}
\usepackage{mathrsfs}
\usepackage[ruled,vlined,linesnumbered]{algorithm2e}
\usepackage{wrapfig}
\usepackage{microtype}
\usepackage{multirow}
\usepackage{lineno}
\usepackage{adjustbox}
\usepackage{booktabs}
\usepackage{subcaption}
\usepackage{amsthm}

\newcommand{\hh}{\mathbf{h}}

\newcommand{\vv}{\mathbf{v}}
\newcommand{\kk}{\mathbf{k}}

\newcommand{\vel}{\mathbf{q}}

\newcommand{\rr}{\mathbf{r}}

\newcommand{\WW}{\mathbf{W}}

\newcommand{\zz}{\mathbf{z}}

\newtheorem{theorem}{Theorem}
\newtheorem{lemma}{Lemma}

\makeatother
\DeclareMathOperator*{\argmin}{arg\,min}

\title{Collaborative Planning for UAV-UGV Teams  \\ under Operational Constraints via \\ Masked Error-bounded Neuro-Symbolic Diffusion}

\author{
    \textbf{Zihao Deng}$^{1}$, 
    \textbf{Qianhuang Li}$^{1}$, 
    \textbf{Maggie Wigness}$^{2}$, \\
    \textbf{John Rogers}$^{2}$\textbf{,} 
    \textbf{Donghyun Kim}$^1$\textbf{,} 
    \textbf{and Hao Zhang}$^1$\vspace{2pt}\\
    $^1$ University of Massachusetts Amherst \\
    $^2$ U.S. Army DEVCOM Army Research Laboratory
}

\begin{document}
\addtocontents{toc}{\protect\setcounter{tocdepth}{-1}}
\maketitle

\begin{wrapfigure}{R}{0.479\textwidth}
\vspace{-13.5pt}
\centering
\includegraphics[width=0.479\textwidth]{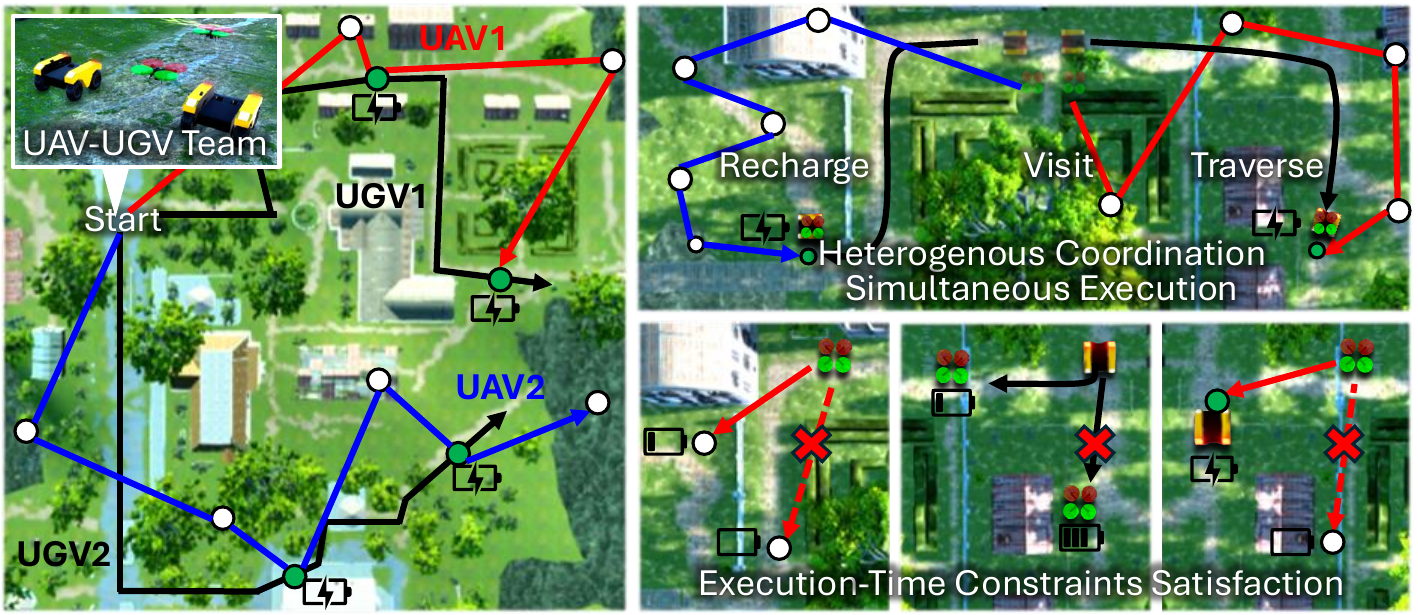}
\vspace{-13.5pt}
\caption{As a collaborative team, traversability-constrained UGVs serve as mobile recharging stations for energy-constrained UAVs, enabling parallel operation while satisfying energy, traversability, and recharging constraints.}
\label{fig:motivation}
\vspace{-10pt}
\end{wrapfigure}
\begin{abstract}
Collaborative planning under operational constraints is essential for heterogeneous robot teams in large-scale missions.
Unmanned Aerial Vehicles (UAVs) provide rapid area coverage but are limited by onboard energy, while Unmanned Ground Vehicles (UGVs) can serve as mobile recharging stations yet are constrained by road-network traversability.
Effective collaboration must enforce operational constraints during execution while enabling parallel, coordinated actions.
Masked Diffusion Models (MDMs) have shown promise through token unmasking but have two key limitations: entropy-guided scheduling assumes distributional alignment and reliance on differentiable constraints for diffusion-time optimization.
We propose \textit{Masked Error-bounded Neuro-Symbolic Diffusion} (MEND), which formulates collaborative planning as a conditional masked generative process
to generate context-aware action sequences for UAV-UGV teams.
MEND introduces a neuro-symbolic diffusion paradigm that constrains conditional masked diffusion policies at execution by reweighting the learned action distribution with a logic-induced admissibility indicator, thereby enforcing compliance with operational constraints.
MEND introduces an error-bounded inference scheduling at execution that constructs the unmasking schedule by selecting positions with minimal predictive entropy over the constraint-aware action distribution, enhancing collaborative planning optimality under operational constraints.
Extensive experimental results in simulation, on real-world maps, and with physical robot teams demonstrate that MEND advances UAV-UGV collaborative planning under operational constraints.
More details are available on the website: \href{https://hcrlab.gitlab.io/project/mend/}{https://hcrlab.gitlab.io/project/mend}.
\end{abstract}


\section{Introduction}

Collaborative planning enables multi-robot systems to coordinate actions and share resources for interdependent tasks, supporting parallel operation \cite{pinciroli2012argos}, scalability \cite{chen2024scalable}, and robustness \cite{gao2023collaborative}.
In applications such as search and rescue \cite{kashyap2025simulations, cai2022multi, queralta2020collaborative, luo2011multi} and environmental monitoring \cite{asadzadeh2022uav, liu2022uav, fascista2022toward, eskandari2020meta}, robot teams are often heterogeneous, comprising platforms with complementary capabilities and constraints. For example, Unmanned Aerial Vehicles (UAVs) provide rapid area coverage but are limited by onboard energy, whereas Unmanned Ground Vehicles (UGVs) offer substantial power reserves and enable in-mission UAV recharging, yet are constrained by road-network traversability.
Consequently, collaborative planning must explicitly enforce operational constraints at execution time.
To ensure efficiency, multi-robot systems must execute plans in parallel with synchronized actions (Figure~\ref{fig:motivation}).

\color{black}

Optimization-based methods, e.g., mixed-integer programming~\cite{manyam2019cooperative, di2017last, climaco2021branch}, and symbolic planners, e.g., based on First-Order Logic (FOL)~\cite{correa2024lifted, lesperance2024abstraction, vazquez2022scheduling}, provide optimal solutions and constraint satisfaction but scale poorly due to combinatorial search. 
Reinforcement learning methods~\cite{mondal2025coordinate} often assume greedy recharging, simplify robot motion, or encode constraints only during training. 
Diffusion models excel in trajectory planning~\cite{chi2023diffusionvisuomotor, carvalho2023motion, xian2023chaineddiffuser, fang2024dimsam}, while Masked Diffusion Models (MDMs) achieve stronger results
via token unmasking \cite{kitouni2024factorization,kim2025any}.
MDMs typically employ entropy-guided inference scheduling~\cite{kim2025train, chen2025optimal} to improve sampling quality. 
However, they assume alignment between the learned and true distributions; when generated actions violate constraints, entropy estimates are distorted, 
leading to suboptimal solutions that breach operational constraints in collaborative planning.
Recent work extends MDMs to constrained generation via classifier-guidance \cite{schiff2024simple} and neuro-symbolic constraint enforcement during sampling \cite{christopher2025neuro, cardei2025constrained}. However, a theoretical limitation is that they rely on gradient-based optimization over differentiable constraints in sampling, whereas operational constraints are often discrete, non-differentiable, and not preserved by standard diffusion dynamics.

We propose \textit{Masked Error-bounded Neuro-Symbolic Diffusion} (\textbf{MEND}), a theoretical approach for collaborative planning in heterogeneous teams under operational constraints. 
MEND represents interdependent operational contexts as a graph,
and formulates collaborative planning as a conditional masked generative process
to generate context-aware action sequences for UAVs and UGVs.
MEND introduces a neuro-symbolic diffusion paradigm that constrains conditional masked diffusion policies at execution time by reweighting the learned action distribution with a FOL-induced admissibility indicator, thereby enforcing compliance with operational constraints at execution.
MEND also introduces an error-bounded inference scheduling strategy at execution that constructs the unmasking schedule by selecting positions with minimal predictive entropy over the constraint-aware action distribution, enhancing collaborative planning optimality under operational constraints.
The joint action sequences generated by MEND are executed by UAVs and UGVs simultaneously, 
which enables parallel operation of heterogeneous robots while satisfying operational constraints.

Our contribution is MEND, a novel and principled learning approach for masked error-bounded neuro-symbolic diffusion, 
to enable collaborative planning for UAV-UGV teams under operational constraints. 
MEND introduces two key novelties:
MEND introduces two key novelties:
\begin{itemize}
    \item \textbf{Execution-time masked neuro-symbolic diffusion} that constrains conditional masked diffusion policies via \textit{neuro-symbolic reasoning}, reshaping the learned action distribution into a constraint-aware distribution using a logic-induced admissibility indicator. 
    This enables the satisfaction of non-differentiable constraints that cannot be optimized by gradient-based diffusion updates, thereby enforcing operational constraints for UAV-UGV teams.
    \item \textbf{Execution-time error-bounded inference scheduling}, which guarantees that predictive entropy computed under the constraint-aware action distribution incurs a provably smaller sampling error relative to the true distribution.
    This reduces misalignment between the learned and true action distributions, thereby improving collaborative planning optimality under operational constraints.

\end{itemize}

\color{black}

\section{Related Work}\label{sec:related}
\paragraph{Collaborative Planning for Multi-Robot Systems}

Optimization-based methods formulate UAV-UGV planning as NP-hard routing problems. Mixed Integer Programming (MIP) \cite{murray2015flying, manyam2019cooperative, di2017last, climaco2021branch}, dynamic programming \cite{liang2024algorithm, kool2022deep, lera2022dynamic}, and heuristic search methods \cite{arnold2019knowledge, gmira2021tabu, yaugmur2021multi, maini2015cooperation, ropero2019terra, mondal2023optimizing} can compute optimal or near-optimal solutions, but their computational cost limits scalability and real-time deployment.

Symbolic planning employs predicate-based representations for constraint-guaranteed reasoning. Prior work uses action language BC \cite{jiang2019multi, zhang2017multirobot}, Planning Domain Definition Language (PDDL) \cite{jovan2023adaptive, mateu2024multi}, Temporal Logic (TL) \cite{sun2022multi, sewlia2025maps2}, and Backus-Naur Form (BNF) \cite{zhang2012iq} to model mission structures and operational constraints in multi-robot systems. These formalisms are largely grounded in First-Order Logic (FOL) \cite{correa2024lifted, lesperance2024abstraction, vazquez2022scheduling}, enabling quantified reasoning for constraint-aware multi-robot planning.

Reinforcement Learning has been applied to routing tasks \cite{kool2018attention, zhang2023coordinated, fuertes2023solving} with energy-constraints \cite{betalo2025dynamic, fan2022deep, brotee2024optimizing}, combined with Graph Learning to encode context \cite{prates2019learning, luo2022graph, boffa2022neural}. 
Neuro-symbolic approaches have emerged for both single-robot \cite{wang2024imperative} and multi-robot systems \cite{wang2025neuro, bai2024twostep, zhang2025lamma, guo2025one} for constraint satisfaction.
Energy-constrained UAVs with UGV support have been explored \cite{mondal2025coordinate}.
However, these methods are often limited by greedy recharging strategies, simplify UGV traversal as straight-line trajectories, and enforce constraints during training rather than providing execution-time guarantees.

\paragraph{Masked Diffusion Models.}
Diffusion models have shown promise in trajectory planning \cite{chi2023diffusionvisuomotor, carvalho2023motion, xian2023chaineddiffuser, fang2024dimsam}.
In discrete domains, masked diffusion models (MDMs) achieve competitive results in task and maze planning \cite{nie2024scaling, gong2024scaling, kitouni2024factorization, ye2024beyond, kim2025any}.
Recent work extends MDMs to constrained generation via classifier-guidance and neuro-symbolic reasoning \cite{schiff2024simple, christopher2025neuro, cardei2025constrained}.
However, these approaches rely on differentiable constraints optimized during diffusion via gradient-based updates, whereas operational constraints are often discrete, non-differentiable, and not preserved by standard diffusion dynamics.

MDMs support flexible inference schedules, whose choice strongly affects sampling quality, as random unmasking can lead to error propagation \cite{shi2024simplified}.
Prior work improves sampling via variable-length generation \cite{kim2025any}, caching \cite{hayakawa2025demystifying}, and entropy-guided scheduling \cite{zheng2024masked, kim2025train}, with recent theory relating optimal scheduling to predictive entropy \cite{chen2025optimal}.
However, existing entropy-guided scheduling assumes alignment between the learned and true data distributions; when constraints are violated, entropy estimates become distorted, resulting in biased unmasking schedules in practice.

\section{Preliminaries}

\textbf{Masked Diffusion Models} (MDMs) \cite{sahoo2024simple} define a discrete-time Markov chain
$a_0 \rightarrow a_1 \rightarrow \cdots \rightarrow a_T$, where $a_0=(a_0^1,\dots,a_0^L)\in\mathcal{A}^L$ is the ground-truth sequence of length $L$ and
$a_t\in(\mathcal{A}\cup\{m\})^L$ for $t\in\{0,\dots, T\}$, with vocabulary $\mathcal{A}$ and mask token $m$.
The forward process progressively corrupts $a_0$ by replacing tokens with $m$. At diffusion step $t$, $a_t$ is sampled from the distribution 
$q(a_t | a_0) = \prod_{i=1}^{L} q(a_t^{i} | a_0^{i})$
and $q(a_t^{i} | a_0^{i})= \mathrm{Cat}\!\left(\alpha_t \, \mathbf{e}_{a_0^{i}} + (1 - \alpha_t)\, \mathbf{e}_{m}\right)$,
where $\alpha_t \in [0,1]$ denotes a predefined noise schedule,
and $\mathbf{e}_{a_0^{i}}$ and $\mathbf{e}_m$ denote the one‑hot vector of $a_0^i$ and $m$, respectively. 
The reverse process iteratively reconstructs $a_0$ from a masked sequence $a_T=(m,\dots,m)$.
For any intermediate step $s < t$, we sample 
$a_s \sim q(a_s \mid a_t, a_0) = \prod_{i=1}^{L} q(a_s^i \mid a_t, a_0)$, where 
$q(a_s^i \mid a_t, a_0) = \mathrm{Cat}(\mathbf{e}_{a_t^i})$ if $a_t^i \neq m$, and otherwise 
$\mathrm{Cat}\!\big(
\frac{1-\alpha_s}{1-\alpha_t}\mathbf{e}_m +
\frac{\alpha_s-\alpha_t}{1-\alpha_t}\mathbf{e}_{a_0^i}
\big)$. Thus, if $a_t^i \neq m$, the token remains unchanged, i.e., $a_s^i = a_t^i$; otherwise, $a_s^i \in \{m, a_0^i\}$ according to $\alpha_s$ and $\alpha_t$.

\paragraph{MDM Inference Scheduling} 
At diffusion step $t$,
given a set of masked positions, 
an \emph{inference schedule} is a mapping $\mathcal{S} : \{T,\dots,1\} \rightarrow 2^{\{1,\dots,L\}}$, where $\mathcal{S}(t) \subseteq \{i \in \{1,\dots,L\} \,|\, a_t^i = m \}$ specifies the subset of masked positions to be revealed. 
The reveal set is updated as $S_{t+1}=S_t\cup\mathcal{S}(t)$, resulting $\{S_t\}_{t=0}^T$ with $S_t \subseteq \{1,\dots, L\}$.
\textit{Uniform scheduling} \cite{swerdlow2025unified} uniformly reveals positions $i_t$ from the remaining masked set, defined as $i_t \sim \mathcal{U}(\{1,\dots, L\}\setminus S_t)$ and $S_{t+1}=S_t\cup\{i_t\}$. 
As the prediction $a_0^i$ depends on $a_0^{S_t}$ in the learned distribution $p(a_0^i | a_0^{S_t}, t)$,
\textit{entropy-guided scheduling} \cite{kim2025train} improves sampling quality by selecting the minimum-entropy position, 
which is defined as $S_{t+1}=S_t\cup\{i_t\}$ with $i_t \in \arg\min_{i\in\{1,\dots, L\}\setminus S_t} H\!\big(p(a_0^i| a_0^{S_t},t)\big)$.

\textbf{Symbolic Logic in Robotic Planning}
can be formalized by a tuple $(\mathcal{O}, \mathcal{A}, \Phi, \mathcal{R})$ \cite{li2025bilevel}, 
where $\mathcal{O}$ is a finite set of \textit{symbols} representing entities, 
$\mathcal{A}$ denotes a set of \textit{grounded actions} with $a \in \mathcal{A}$ corresponding to an instantiated operator that induces a discrete decision, and $\Phi$ denotes a set of \textit{grounded predicates} $\phi: \mathcal{O} \rightarrow \{0,1\}$ encoding entity properties or relations. 
The \textit{logic rule} set $\mathcal{R}$ consists of symbolic action constraints, where each rule $r \in \mathcal{R}$ is a formal formula defined over $\Phi$ to specify preconditions on $\mathcal{A}$.
In this work, $\mathcal{R}$ is instantiated using First-Order Logic (FOL)~\cite{barwise1977introduction,correa2024lifted}, such that each rule takes the form $r: \phi(o) \Rightarrow a(o)$, where $\phi(o)$ is a predicate formula constructed from $\Phi$ using logical connectives $(\wedge, \vee, \neg, \Rightarrow)$ and quantifiers $(\forall, \exists)$ over $o \in \mathcal{O}$, thereby defining the admissible set of $a$ via constraint satisfaction conditions.

\section{Approach}
\paragraph{Problem Definition}

Let $\mathcal{U}^a = \{u^a_i\}$ denote a team of energy-constrained UAVs tasked with visiting a set of task location points $\mathcal{C} = \{c_j\}$. 
Let $\mathcal{U}^g = {u^g_i}$ denote a team of traversability-constrained UGVs that plan trajectories over traversable terrain to rendezvous points $\mathcal{P} = {p_k}$ for UAV recharging.
Then, the symbol set $\mathcal{O} = \mathcal{U}^a \cup \mathcal{C} \cup \mathcal{U}^g \cup \mathcal{P}$ captures all robotic and task-relevant entities.

We represent the operational context as a heterogeneous graph $\mathcal{G} = \{\mathcal{T}, \mathcal{E}^{\mathrm{intra}}, \mathcal{E}^{\mathrm{uav}}, \mathcal{E}^{\mathrm{ugv}}\}$, where the node set is defined as $\mathcal{T}=\mathcal{T}^{c}\cup\mathcal{T}^{p}\cup\mathcal{T}^{a}\cup\mathcal{T}^{g}$.
Task nodes $\mathcal{T}^{c}=\{\mathbf{t}_i^c\}$ encode tasks $c_i\in\mathcal{C}$ with $\mathbf{t}_i^c=[\mathbf{d}_i^c,\delta_i]$, where $\mathbf{d}_i^c\in\mathbb{R}^d$ denotes task coordinates and $\delta_i\in\{0,1\}$ indicates visitation status.
Path nodes $\mathcal{T}^{p}=\{\mathbf{t}_i^p\}$ encode paths $p_i\in\mathcal{P}$ with $\mathbf{t}_i^{p}$ denotes the centroid of $p_i$. 
UAV nodes $\mathcal{T}^{a}=\{\mathbf{t}_i^a\}$ encode UAVs $u_i^a\in\mathcal{U}^a$ with $\mathbf{t}_i^a=[\mathbf{d}_i^a,f_i]$, where $\mathbf{d}_i^a$ and $f_i$ denote UAV coordinates and remaining energy, respectively.
UGV nodes $\mathcal{T}^{g}=\{\mathbf{t}_i^g\}$ encode UGVs $u_i^g\in\mathcal{U}^g$ with $\mathbf{t}_i^g=[\mathbf{d}_i^g,\mathbf{d}^a]$, where $\mathbf{d}_i^g$ denotes UGV coordinates and $\mathbf{d}^a$ represents the coordinates of all UAVs.
The edge sets model heterogeneity and spatial proximity.
$\mathcal{E}^{\mathrm{intra}} =  \{e^{\mathrm{intra}}_{i,j}\}$ connects $\mathbf{t}_i\in\mathcal{T}^k$ and $\mathbf{t}_j\in\mathcal{T}^k$ with $k\in\{c,p,a,g\}$. 
$\mathcal{E}^{\mathrm{uav}} = \{ e^{\mathrm{uav}}_{i,j} \}$ connects $\mathbf{t}_i\in\mathcal{T}^a$ to $\mathbf{t}_j\in\mathcal{T}^c\cup\mathcal{T}^g$.
$\mathcal{E}^{\mathrm{ugv}} = \{ e^{\mathrm{ugv}}_{i,j} \}$ connects 
$\mathbf{t}_i\in\mathcal{T}^g$ to $\mathbf{t}_j\in\mathcal{T}^a\cup\mathcal{T}^p$.
For all edges, $e_{i,j}=1$ if $\|\mathbf{d}_i-\mathbf{d}_j\|_2<\rho$, where $\rho$ is a predefined radius.

We define the \textit{grounded action} set $\mathcal{A}$ for the UAV-UGV team with three action types.
The UAV action $\texttt{Visit}(u^a_i, c_j)$ denotes $u^a_i$ visiting $c_j$, yielding $|\mathcal{U}^a| \times |\mathcal{C}|$ actions.
The UGV action $\texttt{Traverse}(u^g_i, p_j)$ denotes $u^g_i$ traversing to $p_j$, yielding $|\mathcal{U}^g| \times |\mathcal{P}|$ actions.
The recharge action $\texttt{Recharge}(u^a_i, u^g_j)$ denotes $u^a_i$ recharging on $u^g_j$, yielding $|\mathcal{U}^a| \times |\mathcal{U}^g|$ actions.

We define the \textit{grounded predicate} set $\Phi$ as a relational abstraction for describing domain-specific properties for the entire team. 
Each predicate is evaluated to a Boolean value based on the current state of the corresponding UAVs, UGVS, tasks, or paths. 
For example, the predicate $\texttt{UAV}(u^a_i)$ indicates whether $u^a_i$ exists, 
and $\texttt{UnderCost}(u^a_i,c_j)$ indicates whether $c_j$ is energy-affordable for $u^a_i$.

We define the \textit{logic rules} set $\mathcal{R}$ to formalize operational constraints using FOL, which defines a sufficient condition for the admissibility of an action. 
For example, the recharging rule,  \textit{``For the $i$-th UAV, if no task is under cost, and there exists a $j$-th UGV that is under cost, 
then a recharge action with the $j$-th UGV exists"}, can be expressed as $\bigl[
\texttt{UAV}(u_i^a)\wedge
\neg\exists c_k(\texttt{Task}(c_k)\wedge\texttt{UnderCost}(u_i^a,c_k))
\wedge
\exists u_j^g(\texttt{UGV}(u_j^g)\wedge\texttt{UnderCost}(u_i^a,u_j^g))
\bigr]
\Rightarrow
\exists u_j^g\,\texttt{Recharge}(u_i^a,u_j^g)$.

We introduce \textbf{MEND}
(\textit{\underline{M}asked \underline{E}rror-bounded \underline{N}euro-Symbolic \underline{D}iffusion}),
a novel learning approach as shown in Figure~\ref{fig:approach}.
Our goal is to address the collaborative planning for UAV-UGV teams under operational constraints, enabling the following two key multi-robot capabilities:
\begin{itemize}

\item \textbf{Heterogeneous Coordination with Simultaneous Execution}:
The joint planning of $\mathcal{U}^a$ and
$\mathcal{U}^g$ to execute an interdependent mission while minimizing the total makespan. 
Plans are executed in parallel,
where $\mathcal{U}^a$ determine task visitation sequences over
$\mathcal{C}$ and $\mathcal{U}^g$ compute routes over $\mathcal{P}$,
enabling continuous task execution and timely in-mission recharging.

    \item \textbf{Execution-Time Constraint Satisfaction}:
    At execution time, the generated joint plans must satisfy all operational constraints symbolically specified by $\mathcal{R}$, including:
    (1) UAV energy constraints 
    (e.g., $\texttt{UnderCost}(u^a_i,c_j)$) for $\texttt{Visit}(u^a_i, c_k)$, 
    (2) UGV traversability constraints 
    (e.g., $\texttt{ClosestUGV}(u^g_i, u^a_j)$) for $\texttt{Traverse}(u^g_i, p_k)$, and 
    (3) UAV-UGV recharging constraints 
    (e.g., $\texttt{UnderCost}(u^a_i,u^g_j)$) for $\texttt{Recharge}(u^a_i, u^g_j)$.
\end{itemize}

\paragraph{MDM for Collaborative Planning of Heterogeneous Robots}

We formulate collaborative UAV-UGV planning as learning the conditional distribution $p(a | \mathcal{H})$ over joint action sequences $a=(a^1,\dots,a^L)$, where $a^l \in \mathcal{A}$. Due to the combinatorial action space, we approximate $p(a | \mathcal{H})$ using a conditional masked diffusion model (MDM) for context-aware joint action generation.

\color{black}

Given the operational context graph $\mathcal{G}$, we use a heterogeneous graph network $\omega$ \cite{hu2020heterogeneous} to compute node embeddings $\mathcal{H}=\{\mathbf{h}_i\}=\omega(\mathcal{G})$ for context-aware planning.
Each node attribute $\mathbf{t}_i$ is projected as $\mathbf{z}_i=\mathbf{W}^{z}\mathbf{t}_i$, where $\mathbf{W}^{z}$ is a learnable matrix.
Then $\mathbf{z}_i$ is transformed into query, key, and value embeddings as
$\vel_i=\WW_q\zz_i$,
$\kk_i=\WW_k\zz_i$, and
$\vv_i=\WW_v\zz_i$, respectively,
where $\WW_q$, $\WW_k$, and $\WW_v$ are learnable matrices.
To encode homogeneous and heterogeneous relationships, attention is applied over edges  $(i,o)\in\mathcal{E}^{\mathrm{intra}}$, $(i,k)\in\mathcal{E}^{\mathrm{uav}}$, and $(i,n)\in\mathcal{E}^{\mathrm{ugv}}$, with attention weights $\alpha_{i,o}=\mathrm{SoftMax}_o(\vel_i^\top\kk_o/\sqrt{d})$, $\beta_{i,k}^{a}=\mathrm{SoftMax}_k(\vel_i^\top\kk_k/\sqrt{d})$, and $\beta_{i,n}^{g}=\mathrm{SoftMax}_n(\vel_i^\top\kk_n/\sqrt{d})$, respectively, where $d$ denotes the embedding dimension.
Then, $\hh_i$ is obtained by attention aggregation with layer normalization $\text{LN}(\cdot)$ as
$\mathbf{h}_i=\mathrm{LN}\!\left(
\sum_{e_{i,o}^{\mathrm{intra}}=1}\alpha_{i,o}\mathbf{v}_{o}
+
\sum_{e_{i,k}^{\mathrm{uav}}=1}\beta_{i,k}^{a}\mathbf{v}_{k}
+
\sum_{e_{i,n}^{\mathrm{ugv}}=1}\beta_{i,n}^{g}\mathbf{v}_{n}
\right)$.

We apply the forward process $q(a_t^{i}| a_0^{i})$ to the ground-truth action sequence $a_0$ and consider the reverse process  $q(a_{s}^{i}| {a}_t, {a}_0)$, which defines the posterior for recovering $a_0$. 
However, $a_0$ is unknown during sampling, and $q(a_{s}^{i}| {a}_t, {a}_0)$ does not explicitly incorporate operational context.
To address this, 
we use a denoising network to approximate the reverse process that predicts the distribution $p_\theta(a_0^i | a_t,t,\mathcal{H})$ over $a^i_0\in\mathcal{A}$,
which is learned by minimizing the negative log-likelihood over masked positions:
$
\mathbb{E}_{{a}_0,\, 
t,
\sim \mathrm{Unif}(\{1, \dots, T\}),
\, a_t \sim q(a_t | a_0)}
[
\sum_{i:a_t^i=m}
- \log p_\theta(a_0^{i} | {a}_t, t, \mathcal{H})
$, where $\mathrm{Unif}(\cdot)$ denotes the uniform distribution.
The learned distribution $p_\theta(a_0^{i}| {a}_t, t, \mathcal{H})$ enables collaborative planning for the UAV-UGV team by generating a context-aware joint sequence $a$ through conditioning on $\mathcal{H}$.

\color{black}

\begin{figure}[!t]
\centering
\includegraphics[width=1\textwidth]{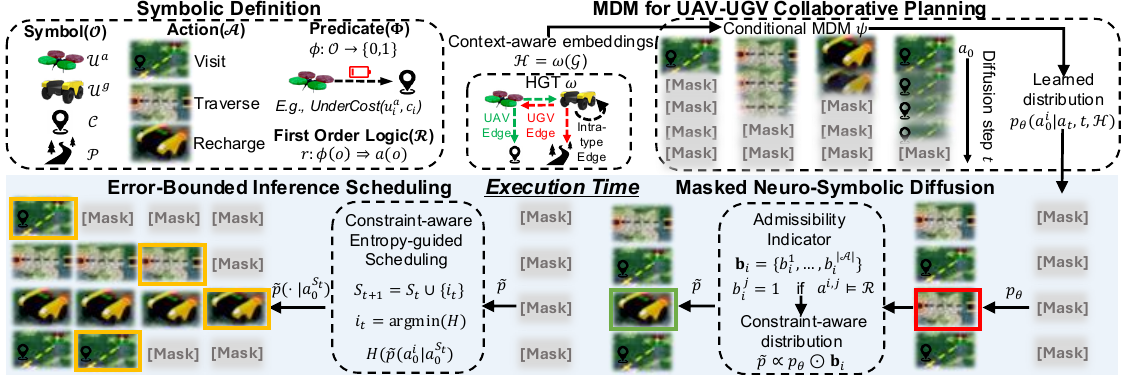}
\caption{MEND formulates collaborative planning as a conditional masked generative process for context-aware joint action generation.
During execution, a neuro-symbolic diffusion paradigm enforces operational constraints via FOL-induced admissibility indicators, while an error-bounded inference scheduling minimizes entropy to improve planning optimality.
}
\label{fig:approach}
\end{figure}

\paragraph{Execution-Time Masked Neuro-Symbolic Diffusion}

As a theoretical limitation, MDMs cannot guarantee $a^i \models \mathcal{R}$ for samples $a^i \sim p_\theta$, since likelihood maximization does not enforce constraints.
Moreover, each FOL rule in $\mathcal{R}$ induces a discrete admissibility indicator, which is non-differentiable and thus cannot be imposed by gradient-based training or standard diffusion sampling.

We propose a neuro-symbolic approach that constrains conditional masked diffusion policies to enforce execution-time constraints for UAV-UGV teams.
MEND reshapes the learned distribution $p_\theta$ into a constraint-aware distribution $\tilde{p}$ by reweighting actions probabilities with an admissibility indicator induced by $\mathcal{R}$.
By removing inadmissible mass while preserving relative likelihoods over admissible actions, MEND satisfies non-differentiable constraints, i.e., $\forall a^i \sim \tilde{p},\; a^i \models \mathcal{R}$.

Let $a_t$ denote a partially masked sequence at step $t$, and 
$\mathcal{M}(a_t)=\{i | a_t^i=m\}$ the masked indices. For each $i \in \mathcal{M}(a_t)$, the MDM defines a distribution over $\mathcal{A}$ as
$p_\theta(a_0^i=a^{i,j}| a_t,t,\mathcal{H}),\ \forall a^{i,j}\in\mathcal{A}.$
This distribution satisfies $p_\theta(\cdot)\ge 0$ and $\sum_{a^{i,j}\in\mathcal{A}} p_\theta(a_0^i=a^{i,j}|\cdot)=1$.
For each position $i$, define an admissibility indicator $\mathbf{b}_i \in \{0,1\}^{|\mathcal{A}|}$, where $b_i^j=1$ iff $a^{i,j}\models\mathcal{R}$.
We construct a constraint-aware distribution by masking and renormalizing as
$\tilde{p}_i =
\frac{p_\theta(\cdot | a_t, t, \mathcal{H}) \odot \mathbf{b}_i}
{\mathbf{1}^\top (p_\theta(\cdot | a_t, t, \mathcal{H}) \odot \mathbf{b}_i)}$.
By construction, $\tilde{p}_i(a^{i,j})=0$ for $a^{i,j}\not\models\mathcal{R}$, and for any admissible actions $a^{i,j},a^{i,k}\models\mathcal{R}$,
$
\frac{\tilde{p}_i(a^{i,j})}{\tilde{p}_i(a^{i,k})}
=
\frac{p_\theta(a^{i,j}| a_t,t,\mathcal{H})}
{p_\theta(a^{i,k}| a_t,t,\mathcal{H})}.
$

Under the recharging rule, if UAV $u_i^a$ cannot afford any \texttt{Visit} action but an affordable UGV $u_j^g$ exists, then \texttt{Recharge}$(u_i^a,u_j^g)$ is the only admissible action at position $i$, with $b_i^{\texttt{Recharge}}=1$ and $b_i^{j}=0$ otherwise. Thus, $\tilde p$ assigns zero probability to $a^{i,j}\not\models\mathcal{R}$, MEND explicitly ensures the satisfaction of execution-time operational constraints under the neuro-symbolic diffusion paradigm.

\paragraph{Execution-Time Error-Bounded Inference Scheduling}

The quality of MDM sampling depends on the schedule $\mathcal{S}$ and revealed sets $\{S_t\}_{t=0}^T$ because $p_\theta(a_0^i | a_0^{S_t}, t, \mathcal{H})$ changes with $a_0^{S_t}$.
Entropy-guided scheduling \cite{kim2025train} selects the minimum-entropy position and assumes the sampled distribution $\nu$ from $p_\theta$ aligns with the ground-truth distribution $\mu$, where $a \sim \mu$ with $a^i \models \mathcal{R}$. 
If $p_\theta$ assigns mass to $a^i \not\models \mathcal{R}$, this entropy deviates from $\mu(a_0^i | a_0^{S_t})$, causing suboptimal, constraint-violating samples.

To address this theoretical limitation, we introduce a novel error-bounded inference scheduling strategy at execution time that guarantees predictive entropy computed under the constraint-aware distribution $\tilde p$ yields a provably smaller mismatch with $\mu(a_0^i | a_0^{S_t})$ than entropy computed under $p_\theta$.

We characterize the effect of $\{S_t\}_{t=0}^T$ on sampling quality by the information curve $\rr=(r_1,\dots,r_L)$, where
$r_j=\mathbb{E}_{|S|=j-1,\;i\notin S}[I(a_0^i;a_0^S)]$ for $j\in\{1,\dots,L\}$, with $S\subseteq\{1,\dots,L\}$ and $i\in\{1,\dots,L\}\setminus S$ sampled uniformly.
Given $S_0=\emptyset$, $S_{t-1}\subseteq S_t$, and $N_t=|S_t|$, the schedule-induced approximation $\rr^T=(r_1^T,\dots,r_L^T)$ is defined as
$r_j^T=r_{N_{t-1}+1}$ for $N_{t-1}<j\le N_t$.
We then quantify the sampling error between the true action distribution $\mu$ and the schedule-induced distribution $\nu_{\{S_t\}_{t=0}^T}$:

\begin{lemma}[Schedule-Induced Sampling Error]
\label{lemma:sample}
Let $\mu$ be the true distribution over $a_0$ with $\mu(a_0^i \not\models \mathcal{R}) = 0$, 
and let $\{S_t\}_{t=0}^T$ be an inference schedule. 
Let $\nu_{\{S_t\}_{t=0}^T}$ denote the distribution induced by 
sampling from $p(\cdot | a_0^{S_t}, t, \mathcal{H})$. 
The schedule-induced sampling error is bounded as \cite{chen2025optimal}:
\begin{equation}
\mathrm{KL}\!\left(\mu \,\|\, \nu_{\{S_t\}_{t=0}^T}\right)
\le
\left\| \mathbf{r} - \mathbf{r}^T \right\|_{1}.
\label{eq:error}
\end{equation}
\end{lemma}

We show that sampling from the constraint-aware distribution $\tilde{p}$ reduces this error through $\tilde{\nu}_{\{S_t\}_{t=0}^T}$:

\begin{table*}[t] 
\centering 
\caption{Example FOL rules for UAV-UGV collaborative planning under operational constraints.}
\vspace{-4pt}
\label{tb:logic_rules} 
\begin{adjustbox}{max width=\textwidth} 
\begin{tabular}{lp{0.85\textwidth}} \toprule \textbf{Rule} & \textbf{First-order logic} \\ \midrule UAV task visitation & \( [\texttt{UAV}(u^a_i)\wedge \exists c_k(\texttt{Task}(c_k)\wedge \texttt{UnderCost}(u^a_i,c_k))] \Rightarrow \exists c_k\,\texttt{Visit}(u^a_i,c_k) \) \\ \midrule UGV path traversal & \( [\texttt{UGV}(u^g_j)\wedge \exists u^a_i\,\exists p_l(\texttt{UAV}(u^a_i)\wedge \texttt{LowEnergy}(u^a_i,f_i)\wedge \texttt{UnderCost}(u^a_i,p_l)\wedge \texttt{ClosestUGV}(u^g_j,u^a_i))] \Rightarrow \exists p_l\,\texttt{Traverse}(u^g_j,p_l) \) \\ \midrule Recharge collaboration & \( [\texttt{UAV}(u^a_i)\wedge \neg \exists c_k(\texttt{Task}(c_k)\wedge \texttt{UnderCost}(u^a_i,c_k)) \wedge \exists u^g_j(\texttt{UGV}(u^g_j)\wedge \texttt{UnderCost}(u^a_i,u^g_j))] \Rightarrow \exists u^g_j\,\texttt{Recharge}(u^a_i,u^g_j) \) \\ \bottomrule \end{tabular} 
\end{adjustbox} 
\vspace{-8pt}
\end{table*}

\begin{theorem}[Error Reduction via Constraint-Aware Scheduling]
\label{thm:constraint}
Let $\nu_{\{S_t\}_{t=0}^T}\!(a^i \models \mathcal{R}) \in (0,1]$ denote the admissible mass under $\nu_{\{S_t\}_{t=0}^T}$; hence $\log \nu_{\{S_t\}_{t=0}^T}\!(a^i \models \mathcal{R}) \leq 0$.
Let $\nu_{\{S_t\}_{t=0}^T}$ and $\tilde{\nu}_{\{S_t\}_{t=0}^T}$ be induced by sampling from $p(\cdot | a_0^{S_t}, t, \mathcal{H})$ and $\tilde{p}(\cdot | a_0^{S_t}, t, \mathcal{H})$, respectively.
Then:
\begin{equation}\small
\mathrm{KL}\big(\mu \;\|\; \tilde{\nu}_{\{S_t\}_{t=0}^T}\big)
=
\mathrm{KL}\big(\mu \;\|\; \nu_{\{S_t\}_{t=0}^T}\big) 
+ \log \nu_{\{S_t\}_{t=0}^T}\!\big(a^i \models \mathcal{R}\big),
\label{eq:kl_difference}
\end{equation}
It follows that constraint-aware scheduling reduces the sampling error:
\begin{equation}
\mathrm{KL}\big(\mu \;\|\; \tilde{\nu}_{\{S_t\}_{t=0}^T}\big)
\le
\mathrm{KL}\big(\mu \;\|\; \nu_{\{S_t\}_{t=0}^T}\big).
\end{equation}
\noindent\textit{Proof.} See Appendix B for the proofs of Lemma~\ref{lemma:sample} and Theorem~\ref{thm:constraint}.
\label{eq:kl_reduction}

\end{theorem}

Lemma~\ref{lemma:sample} shows that the sampling error is governed by the information-curve mismatch $\|\rr-\rr^T\|_1$. Thus, an effective schedule $\{S_t\}_{t=0}^T$ reveals positions with larger mutual information $I(a_0^i;a_0^S)$ earlier to minimize $\|\rr-\rr^T\|_1$. Since $I(a_0^i;a_0^S)=H(a_0^i)-H(a_0^i|a_0^S)$ and $H(a_0^i)$ is fixed, maximizing $I(a_0^i;a_0^{S_t})$ is equivalent to minimizing $H(a_0^i|a_0^{S_t})$. However, if $p_\theta(a_0^i|a_0^{S_t},t,\mathcal{H})$ assigns nonzero mass to inadmissible actions $(a^i\not\models\mathcal{R})$, then $\nu_{\{S_t\}_{t=0}^T}$ may deviate from $\mu$.

Theorem~\ref{thm:constraint} addresses this limitation by constructing $\{S_t\}_{t=0}^T$ under the constraint-aware distribution $\tilde p$, which is provably closer to $\mu$ and satisfies $\tilde p(a^i)=0$ for $a^i\not\models\mathcal{R}$ while preserving the relative likelihoods among admissible actions. 
Since $\mu(a_0^i|a_0^{S_t})$ is unknown at inference time, $H(a_0^i | a_0^{S_t})
= - \sum_{a_0^i} \mu(a_0^i | a_0^{S_t})
\log \mu(a_0^i | a_0^{S_t})$ is intractable. 
MEND approximates it using $H(\tilde p(a_0^i|a_0^{S_t},t,\mathcal{H}))$ and constructs $\{S_t\}_{t=0}^T$ via
$
S_{t+1}=S_t\cup\{i_t\},
$
where
$
i_t\in
\argmin_{i\in\{1,\dots,L\}\setminus S_t}
H(\tilde p(a_0^i|a_0^{S_t},t,\mathcal{H})).
$

As a result, this error-bounded schedule constructed under $\tilde p$ reduces the 
misalignment between $p_\theta$ and the true conditional distribution 
$\mu(a_0^i | a_0^{S_t})$.
Then, the joint action sequence $a$ is generated by sampling $a^i \sim \tilde{p}(a_0^i | a_0^{S_t}, t, \mathcal{H})$, which is executed by UAVs and UGVs simultaneously, enabling parallel operation of heterogeneous robots while satisfying operational constraints.

\section{Experiments}\label{sec:exp}
\paragraph{Experimental Setups}

We evaluate MEND in a high-fidelity Unity-based 3D multi-robot simulator in ROS, using real-world maps reconstructed from 3D LiDAR scans, and on physical robot teams. Mission complexity ranges from 15 tasks with 5 paths to 45 tasks with 10 paths, and team size from 1 UAV-1 UGV to 4 UAV-2 UGV. 
Experiments are conducted in a $10,\mathrm{km} \times 10,\mathrm{km}$ region, where UAVs and UGVs operate at fixed speeds of $q_a=10$~m/s and $q_g=4.5$~m/s, with energy capacities of $287.7$~kJ and $36{,}810$~kJ, respectively.
UAV and UGV energy consumption are modeled as
$0.0461(q_a)^3 - 0.5834(q_a)^2 - 1.8761 q_a + 229.6$ and $464.8 q_g + 356.3$, respectively \cite{hurwitz2021mobile}.
In all scenarios, UAVs visit tasks while UGVs provide recharging support along road networks under operational constraints. To obtain expert data, we adopt MIP for constrained UAV task assignment and UGV path planning. While effective offline, MIP is unsuitable for real-time execution. See Appendices D and E for implementation and training details.

We compare MEND against three metaheuristics planners and two baselines:
(1) \textit{Guided Local Search} (\textbf{GLS}) \cite{arnold2019knowledge}, which augments the routing with penalties to bias the search away from repeatedly selected task assignments;
(2) \textit{Tabu Search} (\textbf{TS}) \cite{gmira2021tabu}, which performs an iterative neighborhood search with a tabu list;
(3) \textit{Simulated Annealing} (\textbf{SA}) \cite{yaugmur2021multi}, which allows early exploration via probabilistic acceptance of worse routes and tightens acceptance over iterations;
(4) \textit{Vanilla MDM} (\textbf{VMDM}): a baseline method %
performs diffusion sampling from the learned distribution using a uniform scheduling; and
(5) \textit{Entropy-guided MDM} (\textbf{EMDM}),
a baseline method (presented in the preliminary based on \cite{kim2025train}), which replaces uniform scheduling with entropy-guided scheduling while sampling from the unconstrained learned distribution. 

We quantitatively evaluate and compare MEND with other methods using three metrics:
(1) \textit{Task Completion Rate} (\textbf{TCR}) \cite{chen2014traccs} is defined as the ratio of completed tasks to all required tasks, measured as a percentage;
(2) \textbf{Makespan} \cite{10207838} is defined as the total time required to complete all task visits, and is reported only for scenarios in which all methods achieve full task completion, measured in seconds;
(3) \textbf{Energy} \cite{qi2025multi} is defined as the cumulative energy consumption of all UAVs and UGVs over the mission horizon, measured in kilojoules.

We present three example First-Order Logic (FOL) rules used in MEND to encode operational constraints, as summarized in Table~\ref{tb:logic_rules}:
(1) \textit{UAV task visitation}, which restricts UAVs to energy-affordable tasks;
(2) \textit{UGV path traversal}, which allows UGV movement only when a UAV requires recharging support; and
(3) \textit{Recharge collaboration}, which enforces that a recharging event occurs when a UAV is about to deplete its energy.

\begin{figure*}[!!!t]
    \centering
    \includegraphics[width=1\textwidth]{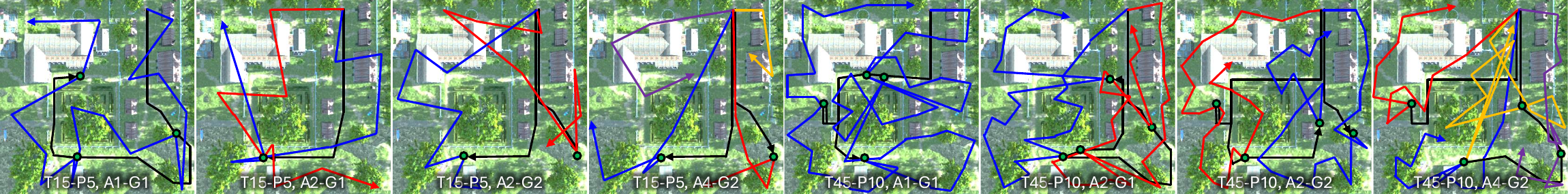}
    \vspace{-14pt}
    \caption{Unity-ROS results with Starling UAVs and Warthog UGVs under different missions. Green circles denote recharge locations, colored curves UAV trajectories, and black lines UGV paths.}
    \vspace{-4pt}
    \label{fig:grid}
\end{figure*}

\begin{table*}[t]
\centering
\caption{Quantitative comparison of MEND against baseline methods, evaluated in simulations.}
\label{tb:combined}
\begin{adjustbox}{max width=\textwidth}
\begin{tabular}{llcccccccccccc}
\toprule
\textbf{Mission} & \textbf{Method} & \multicolumn{3}{c}{\textbf{1UAV--1UGV}} & \multicolumn{3}{c}{\textbf{2UAV--1UGV}} & \multicolumn{3}{c}{\textbf{2UAV--2UGV}} & \multicolumn{3}{c}{\textbf{4UAV--2UGV}} \\
\cmidrule(lr){3-5} \cmidrule(lr){6-8} \cmidrule(lr){9-11} \cmidrule(lr){12-14}
& & TCR & Makespan & Energy & TCR & Makespan & Energy & TCR & Makespan & Energy & TCR & Makespan & Energy \\
\midrule
\multirow{5}{*}{T15-P5} 
& GLS \cite{arnold2019knowledge} & \textbf{100.00} & 8085.33 & 12989.03 & \textbf{100.00} & 6024.59 & 12729.46 & \textbf{100.00} & 5833.05 & 17161.64 & \textbf{100.00} & 4571.07 & 17878.57 \\
& TS \cite{gmira2021tabu}  & \textbf{100.00} & 7664.21 & 12634.75 & \textbf{100.00} & 5673.36 & 13409.86 & \textbf{100.00} & 4841.66 & 15068.19 & \textbf{100.00} & 4874.15 & 19424.62 \\
& SA \cite{yaugmur2021multi} & \textbf{100.00} & 8830.71 & 12287.40 & \textbf{100.00} & 6700.42 & 14045.45 & \textbf{100.00} & 5418.23 & 15704.10 & \textbf{100.00} & 5268.51 & 16712.98 \\

\cmidrule(lr){2-14}

& VMDM (baseline) & 92.32 & 7355.55 & 8120.37 & 90.33 & 5478.12 & 9222.99 & 81.33 & 4050.49 & 8515.34 & 89.83 & 3412.35 & 9420.57 \\
& EMDM (baseline) & 97.72 & 6917.06 & 6543.11 & 98.17 & 4395.18 & 9744.47 & 97.33 & 3955.99 & 6909.90 & 93.50 & 3190.77 & 8350.71\\ 
& MEND (ours)  & \textbf{100.00} & \textbf{6650.30} & \textbf{5893.97} & \textbf{100.00} & \textbf{4026.77} & \textbf{9079.82} & \textbf{100.00} & \textbf{3917.65} & \textbf{6878.65} & \textbf{100.00} & \textbf{2772.57} & \textbf{6966.35}\\
\midrule
\midrule
\multirow{5}{*}{T45-P10} 
& GLS \cite{arnold2019knowledge} & \textbf{100.00} & 11702.12 & 14997.32 & \textbf{100.00} & 10655.58 & 26200.28 & \textbf{100.00} & 9105.88 & 23768.90 & \textbf{100.00} & 8941.05 & 27142.22 \\
& TS \cite{gmira2021tabu}  & \textbf{100.00} & 12502.38 & 15908.72 & \textbf{100.00} & 9401.59 & 21203.38 & \textbf{100.00} & 9770.71 & 27174.47 & \textbf{100.00} & 8091.14  & 28972.64 \\
& SA \cite{yaugmur2021multi} & \textbf{100.00} & 13271.41 & 16731.31 & \textbf{100.00} & 12669.44 & 27797.87 & \textbf{100.00} & 8396.62 & 22683.09 & \textbf{100.00} & 8855.78 & 33757.12 \\
\cmidrule(lr){2-14}
& VMDM (baseline)  & 62.56 & 10184.85 & 10496.56 & 70.28 & 6356.85 & 14536.87 & 73.83 & 5518.29 & 10206.61 & 72.44 & 4716.66 & 13771.02\\
& EMDM (baseline) & 89.61 & 10157.86 & 7826.26 & 78.67 & 6093.69 & 10932.93 & 77.16 & 5506.65 & 9773.20 & 82.59 & 4601.49 & 12435.69 \\
& MEND (ours) &\textbf{100.00} & \textbf{10108.99} & \textbf{7502.81} & \textbf{100.00} & \textbf{5957.50} & \textbf{10555.20} & \textbf{100.00} & \textbf{5290.29} & \textbf{9110.07} & \textbf{100.00} & \textbf{4555.97} & \textbf{12114.19}\\
\bottomrule
\end{tabular}
\end{adjustbox}
\end{table*}

\begin{figure*}[!!!t]
    \centering
    \includegraphics[width=1\textwidth]{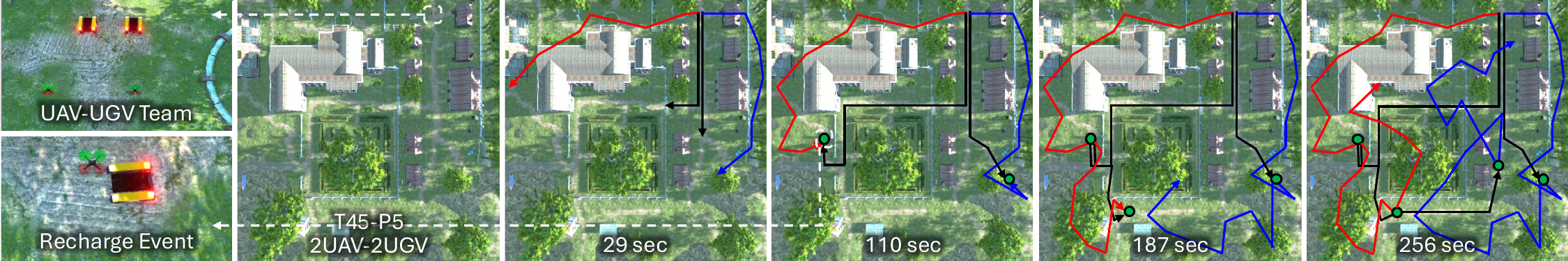}
    \vspace{-14pt}
    \caption{Results for collaborative planning under operational constraints. The first column shows the team configuration and a recharge event, while the remaining columns show different timestamps. }
    \label{fig:flood2}
\end{figure*}

\begin{figure*}[t]
    \centering
    \vspace{-4pt}
    \includegraphics[width=1\textwidth]{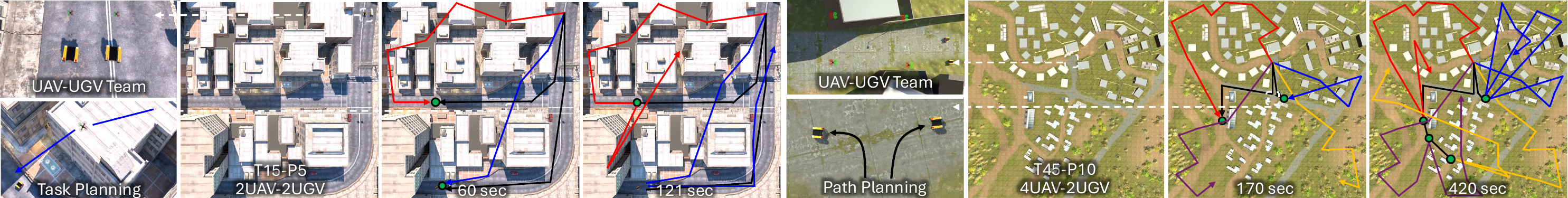}
    \vspace{-14pt}
    \caption{Qualitative results in unseen environments with different depots and road networks: a city scenario with 2 UAV-2 UGV and 15 tasks, and a suburban scenario with 4 UAV-2 UGV and 45 tasks.
}
    \vspace{-4pt}
    \label{fig:gene}
\end{figure*}

\paragraph{Results under Various Task and Team Configurations}
We evaluate MEND in ROS-integrated Unity 3D simulations with energy-constrained Starling-2 UAVs and terrain-constrained Warthog UGVs across varying missions and team sizes.
Figure~\ref{fig:grid} demonstrates scalable collaborative planning under operational constraints, Figure~\ref{fig:flood2} presents a time-stamped scenario illustrating timely in-mission recharging, and Table~\ref{tb:combined} reports quantitative results.
In the 15-task mission, metaheuristic methods achieve perfect TCR but incur the highest makespan and energy due to suboptimal coordination and routing.
Diffusion baselines improve efficiency, with EMDM outperforming VMDM, but still suffer TCR degradation from constraint violations.
MEND achieves 100$\%$ TCR with the lowest makespan and energy across all settings.
In the 45-task mission, MEND further widens the performance gap, demonstrating its effectiveness under operational constraints.

\paragraph{Results on Unseen Simulation and Real-World Environments}

We evaluate the generalizability of MEND in unseen environments with different depots and road networks.
Trained on the rural map in Figure~\ref{fig:grid}, MEND is evaluated on two additional maps in Figure~\ref{fig:gene}, featuring novel road layouts (e.g., a city map with two distinct paths) and depots (e.g., a suburban map with a central team start).
We further evaluate MEND on an unseen 3D-scanned LiDAR campus map, as shown in Figure~\ref{fig:umass}. The 3D-scanned environment enables evaluation under real-world deployment conditions using physical robot models, where the Starling UAV has a $106.6\,\mathrm{kJ}$ capacity with $198.6\,\mathrm{kJ/h}$ consumption at $10\,\mathrm{m/s}$, and the Warthog UGV has a $36{,}810,\mathrm{kJ}$ capacity with $2{,}680.3,\mathrm{kJ/h}$ consumption at $5\,\mathrm{m/s}$ (See Appendix F for additional experimental results). 
These results demonstrate that MEND generalizes effectively to unseen simulations and real-world environments.

\begin{figure*}[!!!t]
    \centering
    \includegraphics[width=1\textwidth]{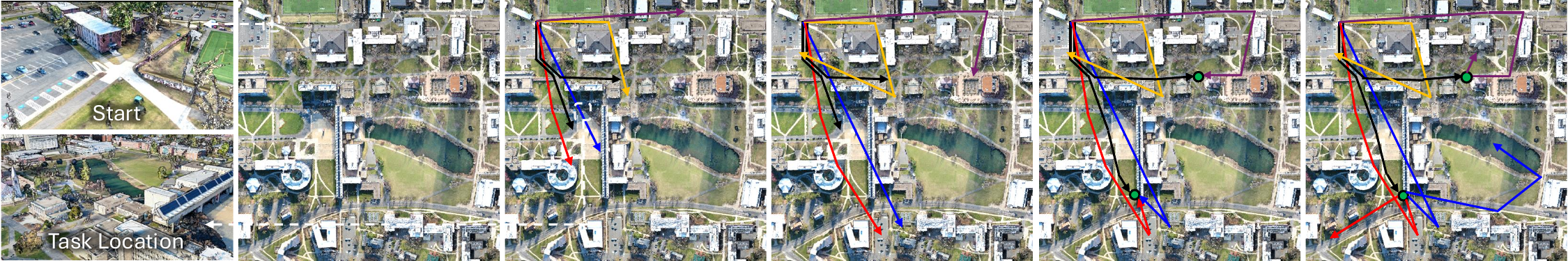}
    \vspace{-14pt}
    \caption{MEND demonstrates generalizability on a 3D LiDAR-scanned map of a real, unseen university campus. This validation scenario consists of 15 tasks executed by 4 UAV-2 UGV.
    }
    \vspace{-4pt}
    \label{fig:umass}
\end{figure*}

\begin{figure*}[t]
    \centering
    \includegraphics[width=1\textwidth]{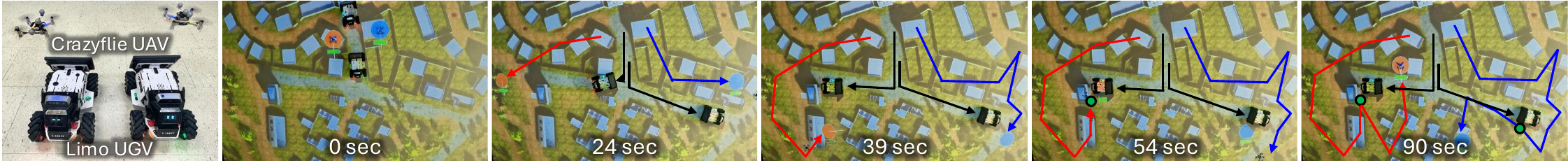}
    \caption{Qualitative results of MEND in real-world mixed-reality case studies using two Crazyflie UAVs and two Limo UGVs controlled via ROS2.
    }
    \label{fig:real}
\end{figure*}

\paragraph{Case Studies on Physical UAV-UGV Teams}
We validate MEND through real-world case studies using a team of two Crazyflie UAVs and two Limo UGVs controlled through ROS2. As shown in Figure~\ref{fig:real}, experiments are conducted in a mixed-reality setup, where projected outdoor environments are overlaid onto the physical workspace and robot poses are tracked via a motion capture system. Tasks are distributed across the map, including regions that are non-traversable for the Limo UGVs. The team successfully and efficiently visits all 15 tasks under operational constraints, with the fourth subfigure illustrating a recharging event and demonstrating MEND's capability for collaborative planning and applicability to real-world deployment.
\paragraph{Scalability}

\begin{wrapfigure}{r}{0.275\columnwidth}
    \centering
    \includegraphics[width=\linewidth]{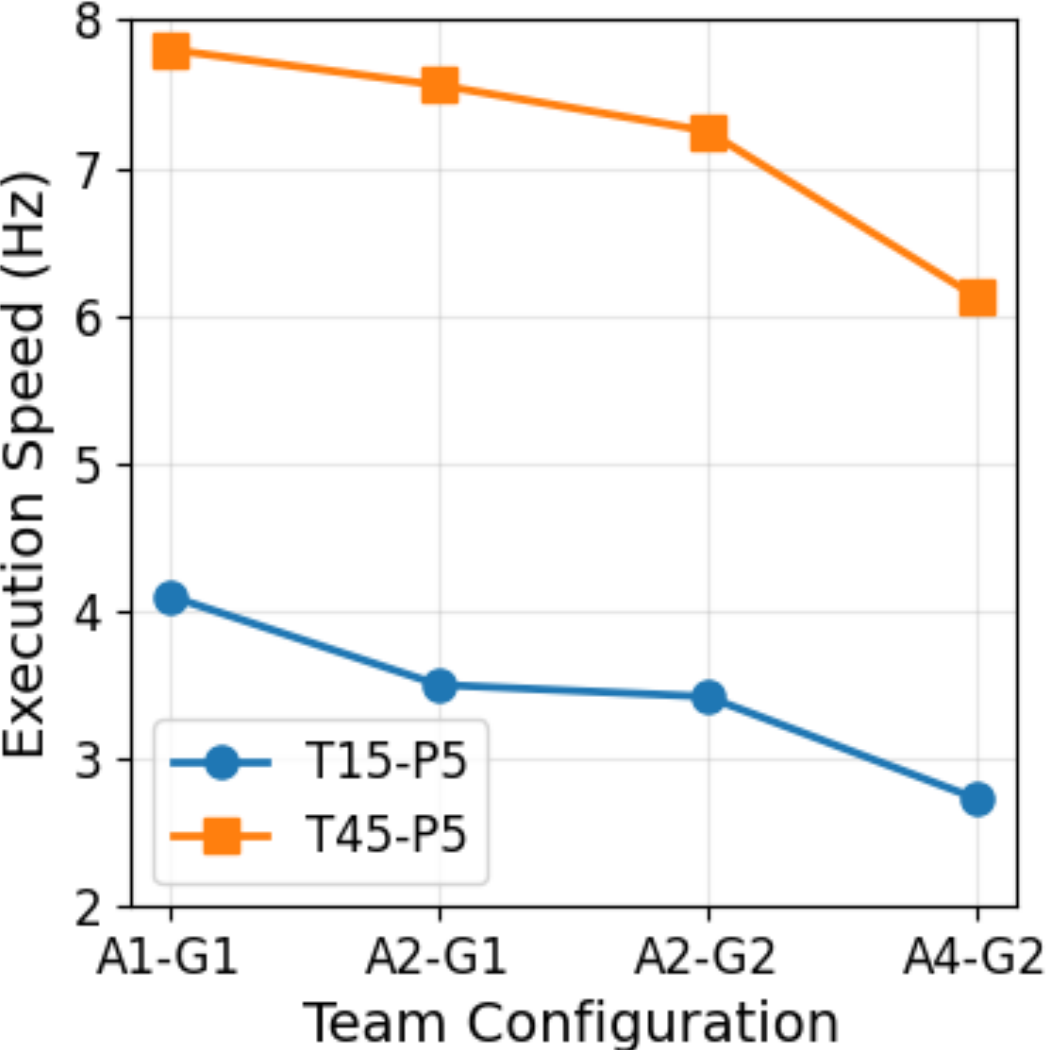}
    \caption{Runtime analysis.}
    \vspace{-10pt}
    \label{fig:execution}
\end{wrapfigure}

We evaluate MEND's scalability across team configurations and mission complexities in Figure~\ref{fig:execution}, which reports execution frequency (Hz) for four team configurations under two mission complexities (see Appendix C for time complexity analysis).
MEND maintains a planning frequency of $7.8$~Hz for A1-G1 with T15-P5 and $2.8$~Hz for A4-G2 with T45-P10.
Compared with ground-truth MIP on A4-G2 with T45-P10, MEND achieves near-optimal makespan and energy, with 4.28\% and 5.95\% gaps, respectively, while running $472\times$ faster.
On T15-P5 with 2 UAVs and 1 UGV, MEND achieves 1.9\% makespan and 3.9\% energy gaps while planning $14.5\times$ faster.
These results demonstrate MEND's scalability to larger UAV-UGV teams.

\begin{wrapfigure}{r}{0.5\columnwidth}
    \vspace{-15pt}
    \centering
    \subfigure[No NS/EGS]{
        \includegraphics[width=0.3\linewidth]{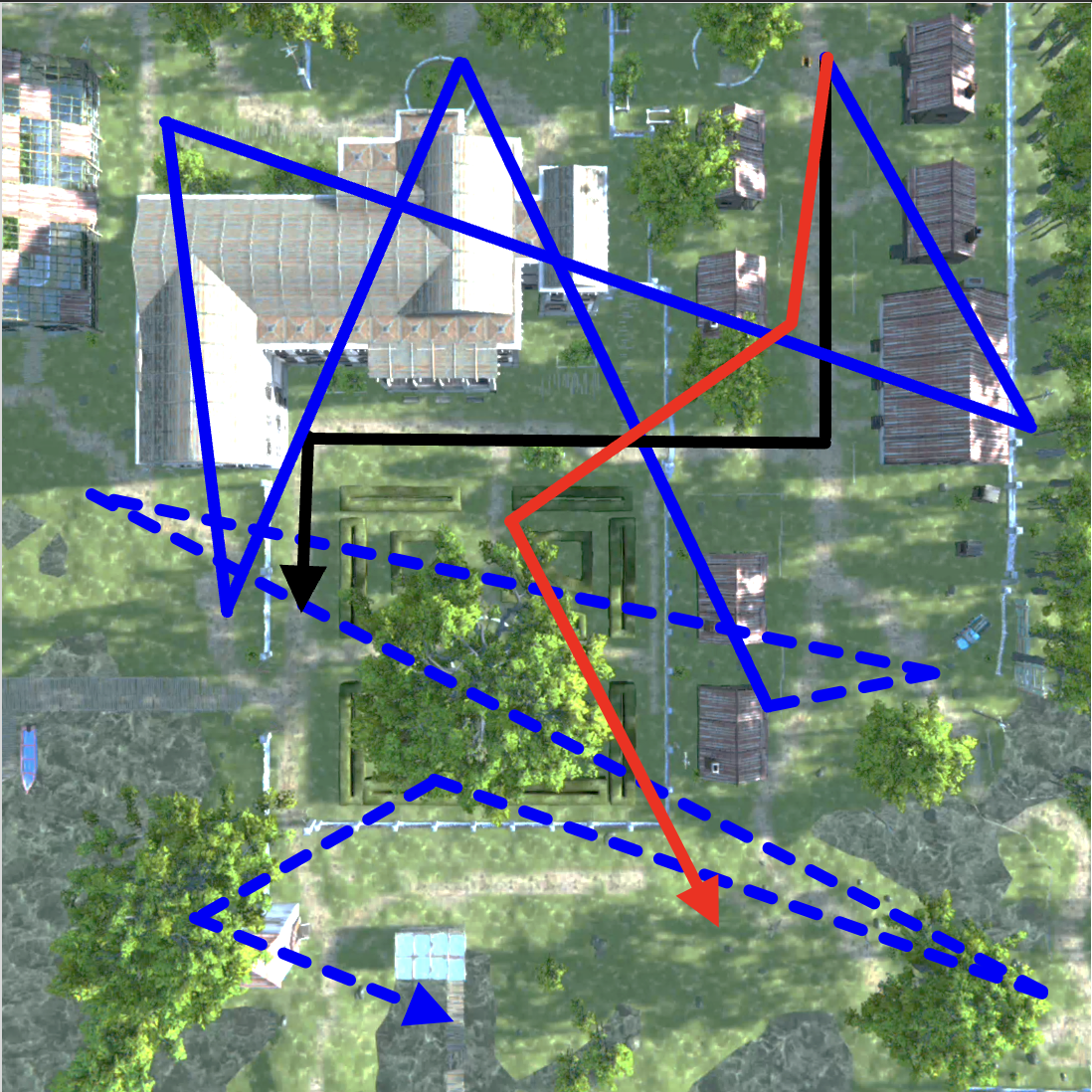}
        \label{fig:vmdm}
    }
    \subfigure[EGS only]{
        \includegraphics[width=0.3\linewidth]{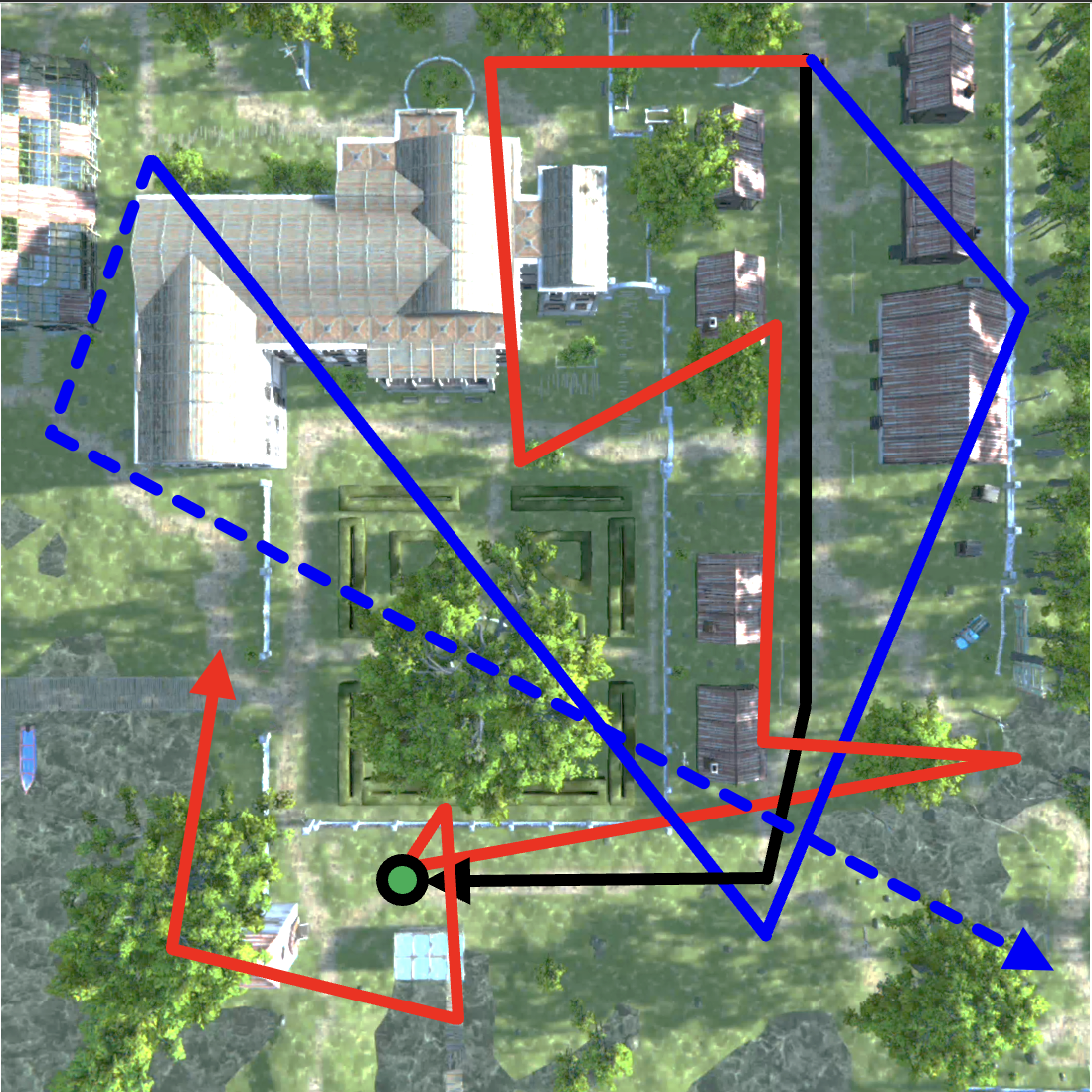}
        \label{fig:emdm}
    }
    \subfigure[Full MEND]{
        \includegraphics[width=0.3\linewidth]{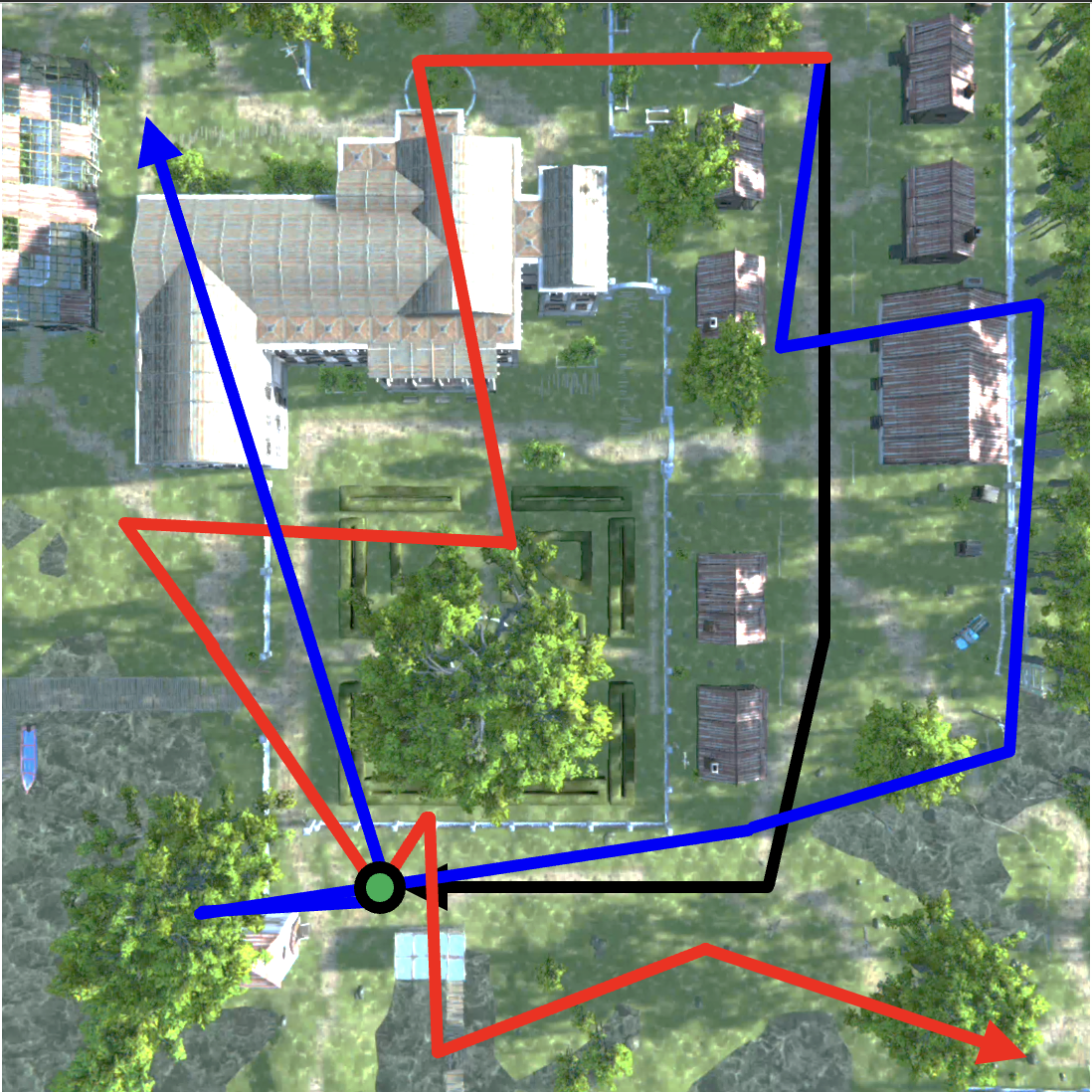}
        \label{fig:mend}
    }
    \vspace{-6pt}
    \caption{Ablation on neuro-symbolic (NS) and entropy-guided scheduling (EGS). Dashed lines denote plans violating operational constraints.}
    \vspace{-10pt}
    \label{fig:comparison}
\end{wrapfigure}

\paragraph{Ablation Study} 
We conduct an ablation study to evaluate the contributions of the neuro-symbolic (NS) module and entropy-guided scheduling (EGS) in MEND, as shown in Figure~\ref{fig:comparison}.
Dashed lines indicate model-predicted plans that violate operational constraints.
Figure~\ref{fig:vmdm} illustrates the variant without NS or EGS, which results in imbalanced UAV workloads and energy violations. Figure~\ref{fig:emdm} shows the variant with EGS but without NS, which improves plan quality yet still fails to satisfy operational constraints.
In contrast, Figure~\ref{fig:mend} shows that the full MEND policy achieves effective collaborative planning while satisfying operational constraints, highlighting the necessity of both NS and EGS.

\section{Limitation}
MEND has two main limitations.
First, MEND plans over a fixed-length masked sequence with a static task set within each planning episode; although replanning can accommodate new tasks, dynamic task arrivals are not explicitly modeled.
Second, while MEND supports decentralized execution via local graph communication, planning relies on joint inference, which potentially limits robustness under communication constraints.
Future work will investigate online task insertion with variable-length masked diffusion models for horizon-agnostic planning and graph-factorized masked diffusion for fully decentralized planning under limited communication.

\section{Conclusion}\label{sec:conclusion}
We propose MEND, a theoretical approach for collaborative planning in heterogeneous robot teams under operational constraints.
MEND formulates collaborative planning as a conditional masked generative process for context-aware planning of both UAVs and UGVs.
It introduces a neuro-symbolic diffusion paradigm that constrains masked diffusion policies to enforce compliance with operational constraints during execution, together with an error-bounded inference scheduling strategy to improve collaborative planning optimality under these constraints.
Extensive experiments in simulation, on real-world maps, and with physical robots demonstrate that MEND advances UAV-UGV collaborative planning under operational constraints. 

\clearpage
\acknowledgments{This work was partially supported by NSF CAREER Award IIS-2308492, DEVCOM ARL TBAM CRA Award W911NF-25-2-0024, DARPA Award N65236-26-9-0001, and NSF Award IIS-2404386.}

\bibliography{example}  %

\end{document}